%% file: main_arxiv.tex
\documentclass{article}
\input{packages}
\input{macros}
\input{arxiv_preamble}

\title{Understanding LLM Parameter Update Sparsity through the Lens of Fisher}
\author{
  Yufan Zhang \qquad Sagnik Mukherjee \qquad Hao Peng  
  \\[2ex]
  University of Illinois Urbana-Champaign
  \\
    \texttt{\{yufanzh, sagnikm3, haopeng\}@illinois.edu}
}

\begin{document}

\maketitle

\input{sections/00_abstract}
    \vspace{-5mm}
\begin{figure}[H]
    \centering
    \includegraphics[width=\linewidth]{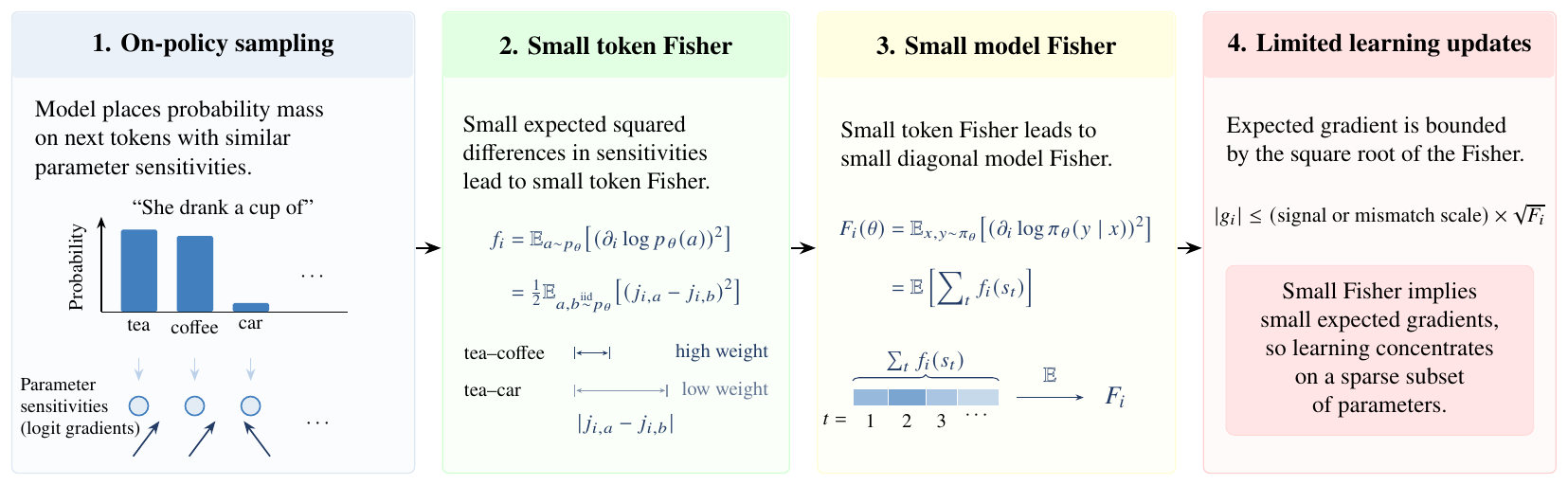}
    \vspace{-6mm}
\caption{Logical chain of how on-policy training gives rise to low Fisher and
concentrated learning.
Section~\ref{sec:why-small-fisher} studies the first two
links, showing how
on-policy sampling can produce small token Fisher and how these token-level
contributions accumulate into small model Fisher.
Section~\ref{sec:fisher_bound_theory} establishes the third link, showing
theoretically how small model Fisher limits expected gradients.
Sections~\ref{sec:test1} and ~\ref{sec:test2} empirically verify whether Fisher identifies
where gradients and effective learning are concentrated.
}
    \label{fig:fisher-chain}
\end{figure}
\input{sections/01_introduction}
\input{sections/02_background_related_work}
\input{sections/03_fisher_concentrated_learning}
\input{sections/04_low_fisher_structure}
\input{sections/05_discussion}

\input{sections/06_conclusion}

\input{sections/07_acknowledgments}

\input{main_arxiv.bbl}
\clearpage
\appendix

\input{sections/appendix_derivations}
\input{sections/appendix_experiment_settings}
\input{sections/appendix_experiment_results}
\end{document}

%% file: packages.tex
\usepackage{times}
\usepackage{fancyhdr}
\usepackage{natbib}
\setcitestyle{authoryear,round,citesep={;},aysep={,},yysep={;}}
\usepackage{amsmath,amssymb}
\usepackage{booktabs,tabularx,graphicx}
\usepackage{float}
\usepackage{amsthm}
\newtheorem{proposition}{Proposition}

\usepackage{tcolorbox}

\newtcolorbox{takeaway}{
  colback=blue!3,
  colframe=blue!35,
  boxrule=0.5pt,
  arc=2pt,
  left=6pt,
  right=6pt,
  top=4pt,
  bottom=4pt,
  before skip=6pt,
  after skip=6pt
}

\newtcolorbox{interpretation}{
colback=yellow!5!white,
colframe=orange!30,
  boxrule=0.5pt,
  arc=2pt,
  left=6pt,
  right=6pt,
  top=4pt,
  bottom=4pt,
  before skip=6pt,
  after skip=6pt
}

\newtcolorbox{statementbox}{
  colback=red!4,
  colframe=red!28,
  boxrule=0.5pt,
  arc=2pt,
  left=6pt,
  right=6pt,
  top=4pt,
  bottom=4pt,
  before skip=6pt,
  after skip=6pt
}

\usepackage{hyperref,url}

%% file: macros.tex
\newcommand{\E}{\mathbb{E}}
\newcommand{\Var}{\operatorname{Var}}
\newcommand{\KL}{D_{\mathrm{KL}}}
\newcommand{\sg}{\operatorname{sg}}

%% file: arxiv_preamble.tex
\makeatletter
\renewcommand{\maketitle}{\par
  \begingroup
    \renewcommand{\thefootnote}{\fnsymbol{footnote}}
    \@maketitle\@thanks
  \endgroup
  \setcounter{footnote}{0}
  \let\maketitle\relax
  \let\@maketitle\relax
  \gdef\@thanks{}\gdef\@author{}\gdef\@title{}\let\thanks\relax}
\renewcommand{\@maketitle}{\vbox{\hsize\textwidth
  {\LARGE\scshape\@title\par}
  \begin{tabular}[t]{l}\bfseries\rule{0pt}{24pt}\@author\end{tabular}
  \vskip 0.3in minus 0.1in}}
\renewenvironment{abstract}
  {\vskip .075in\centerline{\large\scshape Abstract}\vspace{0.5ex}\begin{quote}}
  {\par\end{quote}\vskip 1ex}

\renewcommand{\section}{\@startsection{section}{1}{\z@}
  {-2.0ex plus -0.5ex minus -.2ex}{1.5ex plus .3ex minus .2ex}
  {\large\scshape\raggedright}}
\renewcommand{\subsection}{\@startsection{subsection}{2}{\z@}
  {-1.8ex plus -.5ex minus -.2ex}{.8ex plus .2ex}
  {\normalsize\scshape\raggedright}}
\renewcommand{\subsubsection}{\@startsection{subsubsection}{3}{\z@}
  {-1.5ex plus -.5ex minus -.2ex}{.5ex plus .2ex}
  {\normalsize\scshape\raggedright}}
\renewcommand{\paragraph}{\@startsection{paragraph}{4}{\z@}
  {1.5ex plus .5ex minus .2ex}{-1em}{\normalsize\bfseries}}
\renewcommand{\subparagraph}{\@startsection{subparagraph}{5}{\z@}
  {1.5ex plus .5ex minus .2ex}{-1em}{\normalsize\scshape}}

\renewcommand{\footnoterule}{\kern-3pt\hrule width 12pc\kern 2.6pt}
\renewcommand{\@listi}{\leftmargin\leftmargini}
\renewcommand{\@listii}{\leftmargin\leftmarginii
  \labelwidth\leftmarginii\advance\labelwidth-\labelsep
  \topsep 2pt plus 1pt minus .5pt
  \parsep 1pt plus .5pt minus .5pt\itemsep\parsep}
\renewcommand{\@listiii}{\leftmargin\leftmarginiii
  \labelwidth\leftmarginiii\advance\labelwidth-\labelsep
  \topsep 1pt plus .5pt minus .5pt
  \parsep\z@\partopsep .5pt plus 0pt minus .5pt\itemsep\topsep}
\renewcommand{\@listiv}{\leftmargin\leftmarginiv
  \labelwidth\leftmarginiv\advance\labelwidth-\labelsep}
\renewcommand{\@listv}{\leftmargin\leftmarginv
  \labelwidth\leftmarginv\advance\labelwidth-\labelsep}
\renewcommand{\@listvi}{\leftmargin\leftmarginvi
  \labelwidth\leftmarginvi\advance\labelwidth-\labelsep}
\belowdisplayskip\abovedisplayskip
\renewcommand{\normalsize}{\@setsize\normalsize{11pt}\xpt\@xpt}
\renewcommand{\small}{\@setsize\small{10pt}\ixpt\@ixpt}
\renewcommand{\footnotesize}{\@setsize\footnotesize{10pt}\ixpt\@ixpt}
\renewcommand{\scriptsize}{\@setsize\scriptsize{8pt}\viipt\@viipt}
\renewcommand{\tiny}{\@setsize\tiny{7pt}\vipt\@vipt}
\renewcommand{\large}{\@setsize\large{14pt}\xiipt\@xiipt}
\renewcommand{\Large}{\@setsize\Large{16pt}\xivpt\@xivpt}
\renewcommand{\LARGE}{\@setsize\LARGE{20pt}\xviipt\@xviipt}
\renewcommand{\huge}{\@setsize\huge{23pt}\xxpt\@xxpt}
\renewcommand{\Huge}{\@setsize\Huge{28pt}\xxvpt\@xxvpt}
\makeatother

\hypersetup{
  pdftitle={Understanding LLM Parameter Update Sparsity through the Lens of Fisher},
  pdfauthor={Yufan Zhang, Sagnik Mukherjee, Hao Peng},
  pdfsubject={Preprint}
}

%% file: sections/00_abstract.tex
\begin{abstract}
Recent studies have observed that parameter changes during language-model
post-training can be concentrated in a small subset of coordinates. This
phenomenon has been reported in reinforcement learning, on-policy distillation,
and supervised fine-tuning on near-policy data. Its recurrence across different post-training paradigms suggests shared structure in training dynamics. In this paper, we
examine this pattern through the diagonal model Fisher, which measures the
sensitivity of the model's output distribution to individual parameters and is
independent of any particular reward or teacher signal. Theoretically, we show
that small diagonal Fisher leads to small expected gradients across a range of training
objectives, providing a common explanation for sparse gradient updates. Empirically, we test this connection in RL and OPD. We find that Fisher identifies where
gradients are concentrated, and fixed sparse masks selected from the initial
Fisher retain a large proportion of the improvement from full training. Finally, we investigate the
mechanisms underlying low Fisher in on-policy training. Our results show that high-probability next tokens tend to have similar parameter sensitivities, contributing to low Fisher. Together, these results establish the diagonal
model Fisher as a unifying perspective linking update sparsity to on-policy
training dynamics in LLM post-training.
\end{abstract}

%% file: sections/01_introduction.tex
\section{Introduction}
\label{sec:intro}

Post-training methods such as supervised fine-tuning (SFT), reinforcement
learning (RL), and on-policy distillation (OPD) are widely used to adapt
pretrained language models to downstream tasks and desired behaviors
\citep{ouyang2022instructgpt,shao2024deepseekmath,deepseekai2025r1,agarwal2024onpolicy}.
Despite different objectives, these methods exhibit a similar pattern in
parameter space: parameter changes can be highly concentrated. Recent work
has found that only a small fraction of parameters undergo substantial changes
during RL, OPD, and SFT on near-policy data, and that sparse subsets identified
from training outcomes can retain much of the benefit of full optimization
\citep{mukherjee2025subnetworks,mukherjee2026we,yu2026dense}.
This recurrence across post-training paradigms raises a fundamental question:

\begin{statementbox}
\textbf{Question.}
\emph{What common structure underlies the concentration of learning in a small
subset of parameters during language-model post-training?}
\end{statementbox}

Prior work has proposed several explanations for update sparsity \citep{mukherjee2025subnetworks,mukherjee2026we, zhu2025path,shenfeld2025razor,
miahi2026understanding}, as discussed
in Section~\ref{sec:background}. These works explain sparsity from different perspectives, but do not provide
a common mechanism linking on-policy sampling to where learning is concentrated
across post-training methods.
In this paper, we study the question through the diagonal
model Fisher (defined in Equation~\ref{eq:intro-fisher}), which measures how sensitive
the model's own response distribution is to individual parameters.
Unlike the gradient induced by a particular training objective, the Fisher
depends only on the model and its response distribution, not on a particular
reward or teacher signal.
One key observation is that
the expected gradient for every parameter coordinate can be bounded in the
generic form (Section~\ref{sec:fisher_bound_theory})
\begin{equation}
\label{eq:intro-bound}
|g_i|
\leq
\text{signal or mismatch scale}\times\sqrt{F_i},
\end{equation}
where \(F_i\) is the diagonal model Fisher and \(g_i\) is the expected gradient for parameter
coordinate \(i\). When the objective-specific signal or mismatch scale is controlled, a small
Fisher therefore implies a small expected gradient. This suggests the following answer:

\begin{statementbox}
\textbf{Answer.}
\emph{Diagonal model Fisher provides an objective-independent indicator for
where learning can be concentrated during language-model post-training.}
\end{statementbox}

\begin{figure}[t]
\centering
\includegraphics[width=0.8\linewidth]{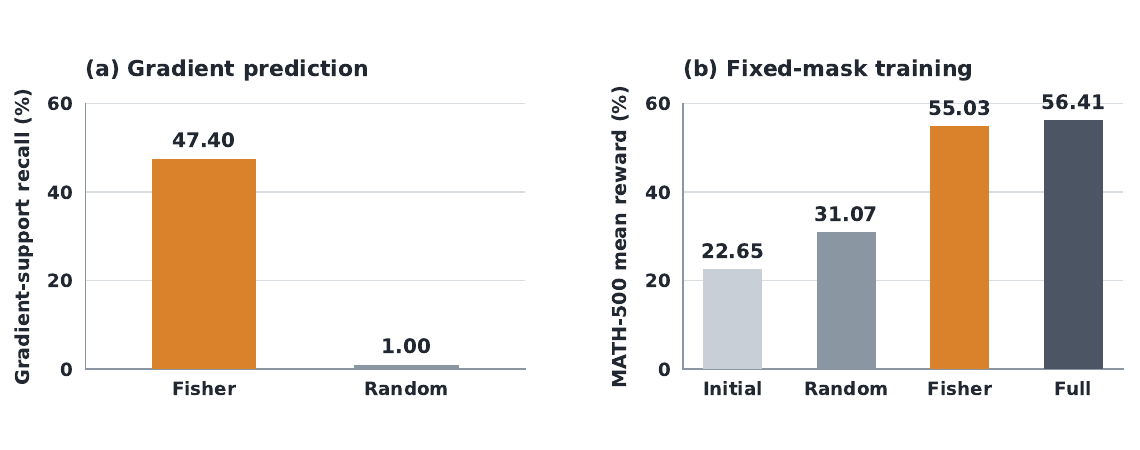}
\vspace{-6mm}
\caption{Fisher identifies a small set of parameters that supports learning.
(a) Before training, the top 1\% Fisher coordinates contain 47.40\% of the
top 1\% gradient coordinates, compared with a 1.00\% random baseline.
(b) Updating only this Fisher-selected subset achieves much of the performance
gain from updating all parameters (Full) and substantially outperforms an
equally sized random subset. Results are for Qwen3-1.7B-Base trained with RL
for mathematical reasoning.}
\label{fig:overview}
\end{figure}

 We then empirically verify
whether Fisher remains informative in actual training. In
Section~\ref{sec:test1}, we find that high-Fisher coordinates capture a much
larger fraction of the largest-gradient coordinates than equally sized random
subsets. In Section~\ref{sec:test2}, updating only parameters with high initial
Fisher while freezing the rest retains a substantial portion of the improvement
achieved by full training. These results empirically support our Fisher-based
explanation of concentrated learning.

We next turn to the low-Fisher pattern itself, asking what gives rise to it in
on-policy training. In Section~\ref{sec:why-small-fisher}, we find that, for
many parameters, the tokens favored by the model tend to respond similarly to
parameter changes. We show that token Fisher is determined by probability-weighted differences
in parameter sensitivities across candidate tokens.
This leads to our central mechanistic explanation for low Fisher in on-policy
training:

\begin{statementbox}
\textbf{Explanation.}
\emph{Fisher is small in on-policy training because the model's probability
mass aligns with next tokens that have similar parameter sensitivities.}
\end{statementbox}

These results bring together the two parts of our analysis: on-policy
generation can give rise to low Fisher through similarity in the parameter
sensitivities of likely next tokens, while low Fisher in turn limits expected
gradients and concentrates learning on a sparse set of parameters. The
overall logical chain is illustrated by Figure~\ref{fig:fisher-chain}, which summarizes the major steps toward our
conclusion.

%% file: sections/02_background_related_work.tex
\section{Background}
\label{sec:background}

\paragraph{Diagonal model Fisher.}
Let \(x\) denote a prompt drawn from a fixed input distribution \(\mathcal D_x\),
\(y\) a model response, and \(\pi_\theta(y\mid x)\) the response distribution
of a model with parameters \(\theta\). For parameter \(\theta_i\),
the diagonal model Fisher is defined as
\begin{equation}
\label{eq:intro-fisher}
F_i(\theta)
=
\E_{x\sim\mathcal D_x}
\E_{y\sim\pi_\theta(\cdot\mid x)}
\left[
\left(
\partial_i\log\pi_\theta(y\mid x)
\right)^2
\right].
\end{equation}
It measures the sensitivity of the model's own response distribution to a
local perturbation of each parameter. For a fixed model and prompt
distribution, it does not depend on the reward, teacher, or demonstration
responses used for post-training.

\paragraph{A common view of post-training objectives.}
We write the training objectives and update surrogates in the common form
\(\max_\theta J(\theta)\), where
\[
J(\theta)
=
\E_{x\sim\mathcal D_x}
\E_{y\sim\mu(\cdot\mid x)}
[R_\theta(x,y)].
\]
Different post-training methods vary in the response distribution \(\mu\)
and the learning signal \(R_\theta\). This common notation allows us to
compare their gradient structure directly. As shown in
Section~\ref{sec:fisher_bound_theory}, despite these differences, their
expected coordinate gradients admit bounds with the same dependence on the
diagonal model Fisher. This common dependence is what allows Fisher to provide
a unified view across post-training objectives.

For autoregressive generation, let \(s_t=(x,y_{<t})\) denote the prefix at
step \(t\) and \(p_\theta(a\mid s_t)\) the corresponding next-token
distribution. We consider SFT, on-policy RL, PG-style OPD, and full-vocabulary
OPD with forward KL, as summarized in Table~\ref{tab:objectives}. Here,
on-policy means that responses are sampled from the current model
\(\pi_\theta(\cdot\mid x)\).

\begin{table}[t]
\centering
\caption{Post-training objectives and update conventions. The demonstration
distribution \(\nu\), teacher \(\pi_T\), and reward \(r\) are fixed, and
\(R_T(x,y)=\sum_t\log[\pi_T(y_t\mid s_t)/p_\theta(y_t\mid s_t)]\).
For PG-style OPD, \(R_T\) is treated as a fixed reward in the
policy-gradient estimator. For full-vocabulary OPD, the rollout distribution
is held fixed during differentiation, and gradients are taken only through
the student's next-token probabilities.}
\label{tab:objectives}
\begin{tabularx}{\linewidth}{l l X}
\toprule
Method & Response distribution \(\mu\) & Learning signal \(R_\theta\) \\
\midrule
SFT & \(\nu(y\mid x)\) & \(\log\pi_\theta(y\mid x)\) \\
On-policy RL & \(\pi_\theta(y\mid x)\) & \(r(x,y)\) \\
PG-style OPD & \(\pi_\theta(y\mid x)\) & \(R_T(x,y)\) \\
Full-vocabulary OPD & \(\pi_\theta(y\mid x)\) &
\(-\sum_t\KL(\pi_T(\cdot\mid s_t)\Vert p_\theta(\cdot\mid s_t))\) \\
\bottomrule
\end{tabularx}
\end{table}

\paragraph{Related work.}

\citet{mukherjee2025subnetworks} show that RL
concentrates updates in small subnetworks whose isolated training can recover
full-training performance. \citet{yu2026dense} report similar observations
for OPD despite its dense teacher supervision. Their subnetwork-retraining
experiments identify coordinates from observed parameter changes. We instead
ask whether Fisher measured before training identifies useful subsets. \citet{rios2025sparsity,adewuyi2026multiple} motivate comparisons with random
masks.

Existing explanations of update sparsity emphasize several mechanisms. \citet{mukherjee2025subnetworks} provide empirical evidence
that sparsity is associated with training on near-policy data.
\citet{zhu2025path} attribute update concentration to pretrained geometry
that favors off-principal updates under a small KL budget, with theoretical
guarantees relying on local curvature and spectral assumptions.
\citet{shenfeld2025razor,miahi2026understanding} show how limited numerical
precision can contribute to update sparsity, while \citet{mukherjee2026we}
 show that optimizer choice strongly affects sparsity.
We use the diagonal model Fisher to bound expected coordinate
gradients across SFT, RL, and OPD, and investigate why Fisher can be small in on-policy training.

Fisher-based parameter selection has also been studied in SFT.
\citet{sung2021fish} select high-Fisher parameters for sparse training,
while \citet{xu2021childtuning} use an empirical Fisher computed from
downstream labels to select trainable parameters. \citet{sharma2024information}
show that Fisher information is concentrated in a small fraction of
language-model parameters and use this structure to guide regularization.
These works use Fisher to guide parameter selection or regularization in
SFT. We instead use Fisher to explain why
learning becomes concentrated, with
fixed-mask training serving as an empirical test of our explanation rather
than as a new parameter-selection method.

%% file: sections/03_fisher_concentrated_learning.tex
\section{Fisher Identifies Where Learning Is Concentrated}
\label{sec:fisher-learning}

\begin{table}[t]
\centering
\small
\caption{Effective signals and squared scales
\(\E_{x,\,y\sim\pi_\theta}[a^2]\) in
Equation~\eqref{eq:generic-fisher-bound}. Here
\(w_\nu(x,y)=\nu(y\mid x)/\pi_\theta(y\mid x)\),
\(\bar r(x)=\E_{\pi_\theta}[r\mid x]\),
\(\bar R_T(x)=\E_{\pi_\theta}[R_T\mid x]\), and
\(A_T(x,y)=\sum_t\bigl(\pi_T(y_t\mid s_t)/p_\theta(y_t\mid s_t)-1\bigr)\).}
\label{tab:fisher-bounds}
\begin{tabularx}{\linewidth}{l X X}
\toprule
Method & Effective signal \(a(x,y)\) & Squared scale \\
\midrule
SFT
&
\(w_\nu(x,y)-1\)
&
\(\E_{x,\,y\sim\pi_\theta}[(w_\nu-1)^2]\)
\\
On-policy RL
&
\(r(x,y)-\bar r(x)\)
&
\(\E_x\Var_{\pi_\theta}(r\mid x)\)
\\
PG-style OPD
&
\(R_T(x,y)-\bar R_T(x)\)
&
\(\E_x\Var_{\pi_\theta}(R_T\mid x)\)
\\
Full-vocabulary OPD
&
\(A_T(x,y)\)
&
\(\E_{x,\,y\sim\pi_\theta}[A_T(x,y)^2]\)
\\
\bottomrule
\end{tabularx}
\end{table}

\subsection{Small Diagonal Fisher Implies Small Gradients}
\label{sec:fisher_bound_theory}

The goal of this section is to make the relationship in
Equation~\eqref{eq:intro-bound} precise. Under the training objectives and update
conventions in Table~\ref{tab:objectives}, each expected coordinate gradient
can be expressed as the model score weighted by an objective-dependent
signal.

Let \(h_i(x,y)=\partial_i\log\pi_\theta(y\mid x)\) denote the score of
parameter \(\theta_i\), so that the diagonal Fisher defined in
Equation~\eqref{eq:intro-fisher} can be written as
\[
F_i
=
\E_{x\sim\mathcal D_x,\,y\sim\pi_\theta(\cdot\mid x)}
[h_i(x,y)^2].
\]
For brevity, we suppress the fixed prompt distribution \(\mathcal D_x\) in
subsequent expectations.

\begin{proposition}[A common Fisher bound]
\label{prop:fisher-gradient-bound}
For each method in Table~\ref{tab:objectives}, its expected update gradient
under the stated differentiation convention can be written as
\begin{equation}
\label{eq:effective-signal}
g_i
=
\E_{x,\,y\sim\pi_\theta}
[a(x,y)h_i(x,y)],
\end{equation}
where the effective signal \(a(x,y)\) depends on the objective and is given
in Table~\ref{tab:fisher-bounds}. Consequently,
\begin{equation}
\label{eq:generic-fisher-bound}
|g_i|
\leq
\sqrt{
\E_{x,\,y\sim\pi_\theta}
[a(x,y)^2]
}
\sqrt{F_i}.\footnote{Similar
Fisher-based gradient bounds are derived by \citet{luo2026rlforgets}.}
\end{equation}
\end{proposition}

\begin{proof}[Proof sketch]
The model score satisfies
\begin{equation}
\label{eq:score-zero}
\E_{y\sim\pi_\theta(\cdot\mid x)}
[h_i(x,y)]
=
\sum_y
\pi_\theta(y\mid x)
\partial_i\log\pi_\theta(y\mid x)
=
0.
\end{equation}
Using this identity together with the method-specific gradient expressions
gives Equation~\eqref{eq:effective-signal}, with derivations provided in
Appendix~\ref{app:fisher-derivations}. Applying Cauchy--Schwarz to this
representation yields Equation~\eqref{eq:generic-fisher-bound}.
\end{proof}

\begin{interpretation}
\textbf{Interpretation.}
All four methods admit the same dependence on \(\sqrt{F_i}\) in the
resulting upper bound. The remaining factor depends on the learning signal
or distribution mismatch.
When this factor is controlled, small diagonal Fisher is a
\textit{sufficient} condition for a small expected gradient for a
\textit{single step}.
\end{interpretation}

Learning over an
entire training trajectory also depends on the optimization dynamics, and
the sparse subset that supports learning need not be unique.
We next test whether this relationship is reflected in where gradients and
effective learning are concentrated in practice. In Section~\ref{sec:test1},
we ask whether Fisher identifies the coordinates on which gradients are
concentrated at a given policy. In Section~\ref{sec:test2}, we move beyond
this local relationship and ask whether a sparse subset identified by Fisher
can sustain learning over an entire training trajectory.

\begin{figure}[t]
\centering
\includegraphics[width=\linewidth]{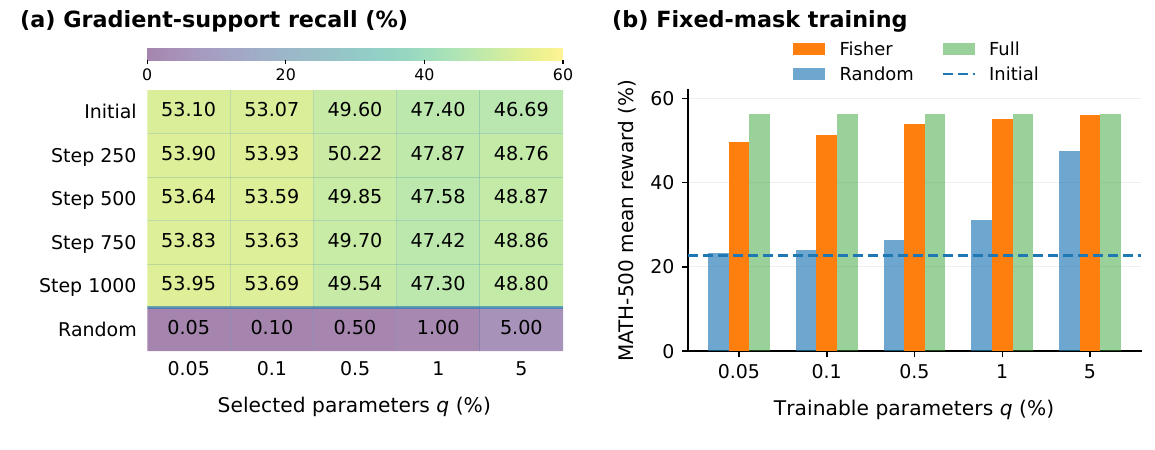}
\vspace{-9mm}
\caption{(a) Gradient-support recall (\%) across selection percentages \(q\).
Initial denotes the pretrained policy. Checkpoint rows average independent
full-parameter training runs. Random is the analytic baseline.
(b) Final MATH-500 mean reward across fixed-mask budgets. Fisher and Random
denote fixed Fisher-selected and uniformly random masks, respectively, while
Full updates all parameters. The dashed line marks the initial policy's score
(\(22.65\%\)).}

\label{fig:fisher-experiments}
\end{figure}

\subsection{Fisher Identifies Where Gradients Are Concentrated}
\label{sec:test1}

The bound in Section~\ref{sec:fisher_bound_theory} implies that, when the objective-dependent scale is controlled, large expected gradients require sufficiently large Fisher. We therefore test whether the largest-gradient coordinates are captured by the largest-Fisher coordinates. For a fixed policy, we rank parameter coordinates separately by Fisher and by
absolute estimated gradient. At a given selection percentage, we then measure
how many of the largest-gradient coordinates are contained in the
highest-Fisher set, using an equally sized random subset as a baseline.

We train Qwen3-1.7B-Base \citep{yang2025qwen3} with Dr.\ GRPO
\citep{liu2025r1zero} for MATH-500 \citep{lightman2023lets}.
Experimental details are given in Appendix~\ref{app:experiment-settings}.

Using the sequence score \(h_i(x,y)\) defined in
Section~\ref{sec:fisher_bound_theory}, we estimate
\begin{equation}
\label{eq:test1-estimators}
\widehat F_i
=
\frac{1}{N_F}
\sum_{n=1}^{N_F}
h_i(x_n,y_n)^2,
\qquad
\widehat g_i
=
\frac{1}{N_gG}
\sum_{n=1}^{N_g}
\sum_{j=1}^{G}
(r_{nj}-\bar r_n)
h_i(x_n,y_{nj}),
\end{equation}
where \(N_F\) and \(N_g\) are the numbers of Fisher and gradient prompts,
\(G\) is the number of responses per gradient prompt, and
\(\bar r_n=G^{-1}\sum_{j=1}^{G}r_{nj}\).

For \(P\) parameter coordinates and a selection percentage \(q\), let
\(K_q=\left\lfloor qP/100\right\rfloor\), and let \(S_F(q)\) and \(S_g(q)\)
denote the globally largest \(K_q\) coordinates of \(\widehat F\) and
\(|\widehat g|\), respectively. We define gradient-support recall by
\begin{equation}
\label{eq:test1-recall}
\operatorname{Recall}(q)
=
100
\frac{|S_F(q)\cap S_g(q)|}{K_q}.
\end{equation}
A uniformly random subset of the same size has expected recall approximately
\(q\%\). As shown in Figure~\ref{fig:fisher-experiments}(a), high-Fisher coordinates are strongly enriched for large-gradient coordinates
at initialization and throughout training across all evaluated cutoffs, which is consistent with our theory. Detailed results and corresponding experiments with OPD are
reported in Appendix~\ref{app:experiment-results}.

\begin{takeaway}
\textbf{Takeaway.}
High-Fisher coordinates capture where the largest gradients are concentrated.
\end{takeaway}

\subsection{Initial Fisher Masks Sustain Learning}
\label{sec:test2}

Gradient-support recall tests the Fisher relationship locally at a fixed
policy, but does not establish whether the identified coordinates remain useful
for learning as the policy changes. In this section, for each parameter budget \(q\), we select the
globally highest-Fisher \(q\%\) of parameters at initialization and keep this
trainable set fixed throughout training. We compare these Fisher-selected
masks with two uniformly random masks of the same size and with full-parameter
training. We use four training seeds for each condition. Experimental details
are given in Appendix~\ref{app:experiment-settings}.

Figure~\ref{fig:fisher-experiments}(b) compares final MATH-500 performance across parameter budgets. The Fisher-selected masks substantially outperform
equally sized random masks across all budgets, while recovering \(85.21\%\) of
the full-training improvement with only \(0.1\%\) of parameters trainable and
\(98.82\%\) at \(q=5\%\).
 Together with the gradient-support results above, these results provide evidence 
that Fisher is informative not only about where gradients are concentrated at
a given policy, but also about which sparse parameter subsets can support
learning over an entire training trajectory. Detailed results and corresponding experiments with OPD are
reported in Appendix~\ref{app:experiment-results}.
\begin{takeaway}
\textbf{Takeaway.}
Initial Fisher identifies sparse parameter subsets that can sustain effective learning throughout training.
\end{takeaway}

%% file: sections/04_low_fisher_structure.tex
\section{Why Does Low-Fisher Structure Arise, and Does It Persist?}
\label{sec:fisher-structure}

The preceding results show that Fisher identifies where learning can be
concentrated. Although on-policy RL and OPD use different learning signals,
both train on responses generated by the current model. This shared sampling
scheme motivates examining the Fisher along the model's own generations. We
therefore ask why Fisher is small under on-policy generation and whether
the resulting low-Fisher structure persists as the policy evolves.

\subsection{Why Is the Fisher Small in On-Policy Training?}
\label{sec:why-small-fisher}

On-policy training uses responses sampled from the model's own distribution.
By definition, a parameter has small Fisher when the log probabilities of
these model-generated responses are less sensitive to changes in that
parameter. Since a response probability is determined by the next-token
probabilities along its generation, we can ask what makes these probabilities
less sensitive. Note that shifting all logits by the same amount leaves the
next-token distribution unchanged. More generally, a parameter has little
effect on the distribution when its perturbation changes the logits of likely tokens
similarly. This suggests a potential explanation for small Fisher: \emph{for many parameters, the logits of
high-probability next tokens may respond similarly to a small change in that
parameter.}

One intuition for this alignment is that likely continuations of
a model-generated prefix may share grammatical, topical, or formatting
requirements. Many parameters may capture these shared requirements and affect several
plausible candidates similarly, while other parameters distinguish between
them \citep{geva2022transformer,gurnee2024universal}. As an illustration, consider the prefix ``She drank a cup of'', for which
``tea'' and ``coffee'' are plausible continuations. A hypothetical parameter direction that increases the model's preference for
liquids could affect the logits of both tokens similarly. Such a direction
would have little effect on the model's relative preference between these
likely continuations. 
We first make this intuition precise by deriving a token-level expression
for Fisher and then empirically test the role of the model's probability assignment.

\paragraph{Token-level Fisher.}
To connect the intuition above to the model Fisher, we relate the Fisher of
an entire response to token-level contributions and express these
contributions through differences in logit sensitivities. \footnote{Related Fisher decompositions for conditional and sequential models are standard in natural-gradient and trajectory-Fisher analyses \citep{martens2020natural,peters2005natural,raykov2026information}.}

For a response \(y=(y_1,\ldots,y_T)\), let \(s_t=(x,y_{<t})\).
Define the token score \(\xi_{t,i}\) for parameter \(\theta_i\)
and the corresponding sequence score \(h_i(x,y)\) as
\[
\xi_{t,i}
=
\partial_i\log p_\theta(y_t\mid s_t),
\qquad
h_i(x,y)
=
\sum_t \xi_{t,i}.
\]
Here, \(p_\theta(\cdot\mid s_t)\) denotes the next-token distribution
at prefix \(s_t\). The \textit{token Fisher} at prefix \(s\) is
\[
f_i(s)
=
\E_{a\sim p_\theta(\cdot\mid s)}
\left[
\left(\partial_i\log p_\theta(a\mid s)\right)^2
\right]
=
\E[\xi_{t,i}^2\mid s_t=s].
\]
Let \(z_a(s)\) denote the logit of token \(a\) at prefix \(s\), write
\(p=p_\theta(\cdot\mid s)\) and
\(p_a=p_\theta(a\mid s)=e^{z_a(s)}/\sum_b e^{z_b(s)}\), and define the
logit sensitivity to \(\theta_i\) as
\(j_{i,a}(s)=\partial_i z_a(s)\).

\begin{proposition}[Token-level decomposition of model Fisher]
\label{prop:token-fisher}
With expectations over the fixed prompt distribution and responses sampled
from \(\pi_\theta(\cdot\mid x)\), the diagonal model Fisher satisfies
\begin{equation}
\label{eq:token-fisher-decomposition}
F_i(\theta)
=
\E\!\left[\sum_t f_i(s_t)\right].
\end{equation}
At each prefix \(s\), the token Fisher can be expressed as
\begin{equation}
\label{eq:token-fisher-pairwise}
f_i(s)
=
\Var_{a\sim p}\!\left[j_{i,a}(s)\right]
=
\frac12
\E_{a,b\overset{\mathrm{iid}}{\sim}p}
\left[
\left(j_{i,a}(s)-j_{i,b}(s)\right)^2
\right].
\end{equation}
\end{proposition}

\begin{proof}[Proof sketch]
Under on-policy generation, each next token is sampled from
\(p_\theta(\cdot\mid s_t)\), so
\begin{equation}
\label{eq:token-score-zero}
\E[\xi_{t,i}\mid s_t]
=
\sum_a p_\theta(a\mid s_t)
\partial_i\log p_\theta(a\mid s_t)
=
0.
\end{equation}
Thus, the token scores form a martingale difference sequence with respect
to the generation history. For \(u<t\), the earlier score \(\xi_{u,i}\)
is determined by \(s_t\), and the tower property gives
\[
\E[\xi_{u,i}\xi_{t,i}]
=
\E\!\left[\xi_{u,i}\E[\xi_{t,i}\mid s_t]\right]
=
0.
\]
The cross-token terms therefore vanish in expectation, yielding
\[
F_i(\theta)
=
\E\!\left[\left(\sum_t\xi_{t,i}\right)^2\right]
=
\E\!\left[
\sum_t\xi_{t,i}^2
+
2\sum_{u<t}\xi_{u,i}\xi_{t,i}
\right]
=
\E\!\left[\sum_t f_i(s_t)\right].
\]
Derivation of Equation~\eqref{eq:token-fisher-pairwise} is
provided in Appendix~\ref{app:token-fisher-pairwise}.
\end{proof}

\begin{interpretation}
\textbf{Interpretation.}
Equation~\eqref{eq:token-fisher-decomposition} shows that the diagonal model Fisher is determined by the token-level contributions
accumulated along the model's own responses.
Equation~\eqref{eq:token-fisher-pairwise} formalizes the intuition above:
token Fisher depends on how differently a parameter
changes the logits of possible next tokens.
It can therefore be small
when the model assigns most probability to tokens with similar sensitivities.
\end{interpretation}

\paragraph{Testing the alignment between probability mass and sensitivities.}
To empirically examine our explanation, we compare the same pairwise quantity under
the model's probabilities and a uniform distribution over next tokens,
while keeping the model and prefix fixed.
 Formally, let \(u\)
be the uniform distribution over the vocabulary of size \(V\), with
\(u_a=1/V\). We define the uniform reference as
\begin{equation}
\label{eq:uniform-token-fisher}
f_i^{\mathrm{unif}}(s)
=
\frac12
\E_{a,b\overset{\mathrm{iid}}{\sim}u}
\left[
\left(j_{i,a}(s)-j_{i,b}(s)\right)^2
\right].
\end{equation}

Observing \(f_i(s)<f_i^{\mathrm{unif}}(s)\) would be consistent with the hypothesis that high-probability tokens tend to have more similar logit sensitivities to \(\theta_i\) than tokens drawn uniformly from the vocabulary.

We perform this comparison at initialization and throughout training, both
over all parameters and within current low-Fisher groups. For each group,
we report its average token Fisher under the model's probabilities and the
corresponding uniform reference, together with their ratio. 
Experimental details are provided in Appendix~\ref{app:experiment-settings}.

\begin{table}[t]
\centering
\small
\setlength{\tabcolsep}{3.5pt}
\caption{Average token Fisher under the model's probabilities (Real), the uniform
reference (Uniform), and their ratio. Fisher values are in units of \(10^{-6}\).
Low-Fisher groups are selected from each checkpoint.
Checkpoint values average four training runs. The initial policy is shared.
}

\label{tab:alignment}
\begin{tabular}{lccccccc}
\toprule
& & & \multicolumn{5}{c}{Current low-Fisher group (\%)} \\
\cmidrule(lr){4-8}
Checkpoint & Metric & All & 99.95 & 99.9 & 99.5 & 99 & 95 \\
\midrule
Initial
& Real    & 22.19 & 4.48 & 3.98 & 3.06 & 2.72 & 1.87 \\
& Uniform & 141.21 & 41.81 & 39.36 & 34.43 & 32.48 & 27.87 \\
& Ratio   & \textbf{0.16} & \textbf{0.11} & \textbf{0.10} & \textbf{0.09} & \textbf{0.08} & \textbf{0.07} \\
\midrule
Step 250
& Real    & 11.57 & 3.00 & 2.71 & 2.13 & 1.90 & 1.32 \\
& Uniform & 184.21 & 46.77 & 43.07 & 35.72 & 32.89 & 26.01 \\
& Ratio   & \textbf{0.06} & \textbf{0.06} & \textbf{0.06} & \textbf{0.06} & \textbf{0.06} & \textbf{0.05} \\
\midrule
Step 500
& Real    & 10.99 & 2.91 & 2.63 & 2.05 & 1.82 & 1.22 \\
& Uniform & 185.50 & 45.87 & 42.36 & 35.35 & 32.59 & 25.86 \\
& Ratio   & \textbf{0.06} & \textbf{0.06} & \textbf{0.06} & \textbf{0.06} & \textbf{0.06} & \textbf{0.05} \\
\midrule
Step 750
& Real    & 10.09 & 2.69 & 2.43 & 1.89 & 1.68 & 1.11 \\
& Uniform & 200.53 & 48.69 & 44.93 & 37.41 & 34.43 & 27.26 \\
& Ratio   & \textbf{0.05} & \textbf{0.06} & \textbf{0.05} & \textbf{0.05} & \textbf{0.05} & \textbf{0.04} \\
\midrule
Step 1000
& Real    & 10.71 & 2.86 & 2.57 & 2.00 & 1.77 & 1.17 \\
& Uniform & 242.76 & 63.16 & 58.05 & 47.88 & 43.89 & 34.20 \\
& Ratio   & \textbf{0.04} & \textbf{0.05} & \textbf{0.04} & \textbf{0.04} & \textbf{0.04} & \textbf{0.03} \\
\bottomrule
\end{tabular}
\end{table}

As shown in Table~\ref{tab:alignment}, the Real-to-Uniform ratios are
consistently below one at initialization and throughout training, both across
all parameters and within the low-Fisher groups. This indicates that, along
the model's own prefixes, high-probability next tokens tend to have more
similar parameter sensitivities, reducing the token-level Fisher and hence the sequence Fisher accumulated
along its generations.

\begin{takeaway}
\textbf{Takeaway.}
In on-policy training, high-probability tokens tend to have more similar sensitivities, which leads to lower Fisher.
\end{takeaway}

\subsection{Stationarity of the Diagonal Fisher}
\label{sec:stationarity}

In this section, we examine whether the initial Fisher pattern persists as
the policy changes during training. This is relevant beyond a single step
because the bounds in Section~\ref{sec:fisher_bound_theory} depend on the
Fisher of the current policy. It is also important for interpreting the
fixed-mask experiments in Section~\ref{sec:test2}, where trainable
coordinates are selected only once at initialization.

We examine whether coordinates remain in their initial low- or high-Fisher
sets and whether initially low-Fisher coordinates remain low in magnitude. Using the Fisher
estimates from Section~\ref{sec:test1}, we compare the initial policy with
later checkpoints of full-parameter training. For each selection percentage \(q\), let
\(K_q=\lfloor qP/100\rfloor\), where \(P\) is the number of parameters.
At step \(k\), let \(H_k(q)\) contain the \(K_q\)
coordinates with the highest estimated Fisher, and let \(L_k(q)\) contain
the remaining coordinates.

We measure how much of each initial set remains in the corresponding set at
step \(k\):
\begin{equation}
\label{eq:fisher-retention}
\operatorname{Retention}^{\mathrm{low}}_k(q)
=
100\,\frac{|L_0(q)\cap L_k(q)|}{P-K_q},
\qquad
\operatorname{Retention}^{\mathrm{high}}_k(q)
=
100\,\frac{|H_0(q)\cap H_k(q)|}{K_q}.
\end{equation}

To measure whether initially low-Fisher coordinates remain low in magnitude, we keep the
initial low-Fisher set \(L_0(q)\) fixed and evaluate its average Fisher at
each checkpoint:
\begin{equation}
\label{eq:fixed-low-fisher-mean}
M_k(q)
=
\frac{1}{|L_0(q)|}
\sum_{i\in L_0(q)}
\widehat F_{k,i}.
\end{equation}
We report \(M_k(q)/M_0(q)\), so that a value near one indicates that the
average Fisher of the initially low-Fisher coordinates remains close to its
initial scale. For comparison, we also report the corresponding ratio over all parameters.

\begin{table}[t]
\centering
\small
\setlength{\tabcolsep}{2pt}
\caption{Persistence of the initial Fisher pattern during full-parameter
training. Retention reports low/high-Fisher set overlap (\%), and
\(M_k/M_0\) reports the Fisher magnitude of the fixed initial low-Fisher set
relative to initialization. Values are averaged over four runs. Random gives
the analytic retention baseline.}
\label{tab:fisher-retention}
\begin{tabular}{@{}lccccccc@{}}
\toprule
& & \multicolumn{6}{c}{Selection \(q\) (\%)} \\
\cmidrule(lr){3-8}
Checkpoint & Metric & All & 0.05 & 0.1 & 0.5 & 1 & 5 \\
\midrule
Step 250
& Retention & -- & 99.99 / 75.76 & 99.98 / 76.96 & 99.89 / 77.67
& 99.76 / 76.70 & 98.78 / 76.89 \\
& \(M_k/M_0\) & 0.70 & 0.92 & 0.92 & 0.93 & 0.93 & 0.92 \\
\midrule
Step 500
& Retention & -- & 99.99 / 73.02 & 99.97 / 74.36 & 99.88 / 75.30
& 99.74 / 74.37 & 98.66 / 74.62 \\
& \(M_k/M_0\) & 0.73 & 0.99 & 1.00 & 1.00 & 1.01 & 0.99 \\
\midrule
Step 750
& Retention & -- & 99.99 / 72.02 & 99.97 / 73.48 & 99.87 / 74.61
& 99.73 / 73.68 & 98.63 / 73.96 \\
& \(M_k/M_0\) & 0.76 & 1.04 & 1.06 & 1.07 & 1.07 & 1.06 \\
\midrule
Step 1000
& Retention & -- & 99.99 / 71.21 & 99.97 / 72.85 & 99.87 / 74.09
& 99.73 / 73.22 & 98.61 / 73.61 \\
& \(M_k/M_0\) & 0.83 & 1.15 & 1.16 & 1.17 & 1.17 & 1.16 \\
\midrule
Random
& Retention & -- & 99.95 / 0.05 & 99.90 / 0.10 & 99.50 / 0.50
& 99.00 / 1.00 & 95.00 / 5.00 \\
\bottomrule
\end{tabular}
\end{table}

As shown in Table~\ref{tab:fisher-retention}, the initial high-Fisher
sets retain roughly \(70\%\)--\(78\%\) of their coordinates at later
checkpoints, far above the corresponding random baselines. The average Fisher of the fixed initial low-Fisher sets remains close to
its initial value through the middle of training and increases by only
about \(15\%\)--\(17\%\) by the final checkpoint.

\begin{takeaway}
\textbf{Takeaway.}
The initial diagonal Fisher pattern remains largely
preserved throughout training in our setting.

\end{takeaway}

%% file: sections/05_discussion.tex
\section{Limitations and Future Work}
\label{sec:discussion}

Our bounds concern expected gradients at a fixed policy rather than
accumulated parameter changes. Extending the analysis to the latter requires
accounting for gradient noise, optimizer dynamics, and learning-rate
schedules. Although our
token-level analysis explains how the model's probability weighting can
reduce Fisher, it does not fully explain how architecture and pretraining
shape the locations of low-Fisher coordinates.

Our masked-training experiments use dense computation and do not establish
memory or runtime savings. Realizing such benefits requires implementations
that exploit the selected masks. Beyond parameter selection, future work
could examine Fisher as a training diagnostic. Extending the study to
multiple tasks and longer training trajectories would also help determine
whether Fisher-selected subsets can support new learning while preserving
previously acquired capabilities.

%% file: sections/06_conclusion.tex
\section{Conclusion}
\label{sec:conclusion}

We studied the concentration of learning during language-model post-training
through the diagonal model Fisher. Across SFT, RL, and OPD, our analysis shows
that small Fisher constrains expected coordinate gradients when the
objective-specific signal or distribution-mismatch scale is controlled.
Empirically, Fisher-selected coordinates are enriched for large gradients,
and fixed masks selected at initialization retain a large proportion of the improvement
from full training in RL and OPD.

Our token-level analysis expresses the model Fisher as the expected sum of
probability-weighted variations in logit sensitivities. This explains how
the model's own probability assignment can contribute to low Fisher when
likely next tokens respond similarly to parameter changes. The observed
persistence of Fisher-selected sets further indicates that the initial
selection remains informative as training progresses. Together, these
findings trace the path illustrated in Figure~\ref{fig:fisher-chain}, from on-policy generation to low Fisher and small
expected gradients for many parameters, helping explain why effective
learning can be concentrated in a sparse parameter subset.

\subsection*{AI use statement}
We used generative AI tools to aid or polish writing, for retrieval and discovery, for research ideation or execution, and to draft sections of the paper.  We have reviewed all AI-assisted work. We take responsibility for the final content of this work with the aid of generative AI.

%% file: sections/07_acknowledgments.tex
\section*{Acknowledgments}

We thank members of the Alta group at the University of Illinois Urbana-Champaign for their helpful feedback, and Yuemin Yu for providing additional compute support.
This work was supported by NSF Grant No. CHE2505932, an Amazon AICE award, and a Capital One ASKS award.
This research also used the Delta advanced computing and data resources, which are supported by the National Science Foundation (award OAC 2005572) and the State of Illinois. Delta is a joint effort of the University of Illinois Urbana-Champaign and its National Center for Supercomputing Applications.
This research used the DeltaAI advanced computing and data resource, which is supported by the National Science Foundation (award OAC 2320345) and the State of Illinois. DeltaAI is a joint effort of the University of Illinois Urbana-Champaign and its National Center for Supercomputing Applications.

%% file: sections/appendix_derivations.tex
\section{Mathematical Derivations}
\label{app:derivations}

\subsection{A Common Fisher Bound}
\label{app:fisher-derivations}

We prove Proposition~\ref{prop:fisher-gradient-bound} under the
differentiation conventions in Table~\ref{tab:objectives}. Throughout,
the prompt distribution is fixed independently of \(\theta\), and
\(\E_\pi\) denotes expectation over prompts and responses
\(y\sim\pi_\theta(\cdot\mid x)\). We assume differentiable probabilities,
that differentiation and expectation can be interchanged, and finite
second moments. For variable-length responses, generation stops at EOS
or a fixed maximum horizon; token scores after termination are zero.

Write \(h_i(x,y)=\partial_i\log\pi_\theta(y\mid x)\). Normalization gives
\begin{equation}
\label{eq:app-score-zero}
\E_\pi[h_i\mid x]
=\sum_y\pi_\theta(y\mid x)\partial_i\log\pi_\theta(y\mid x)
=\partial_i\sum_y\pi_\theta(y\mid x)=0.
\end{equation}
Thus any response-independent baseline \(b(x)\) satisfies
\(\E_\pi[b(x)h_i]=0\). Once \(g_i=\E_\pi[a h_i]\) is established,
Cauchy--Schwarz yields
\begin{equation}
\label{eq:app-common-bound}
|g_i|
\leq\sqrt{\E_\pi[a^2]}\sqrt{\E_\pi[h_i^2]}
=\sqrt{\E_\pi[a^2]}\sqrt{F_i}.
\end{equation}
It remains to identify \(a\) for the four objectives.

\paragraph{Supervised fine-tuning.}
Let the fixed demonstration distribution \(\nu(\cdot\mid x)\) be
absolutely continuous with respect to \(\pi_\theta(\cdot\mid x)\).
With \(w_\nu(x,y)=\nu(y\mid x)/\pi_\theta(y\mid x)\),
\begin{align*}
g_i
&=\partial_i\E_{x,y\sim\nu}[\log\pi_\theta(y\mid x)]
=\E_{x,y\sim\nu}[h_i]\\
&=\E_\pi[w_\nu h_i]
=\E_\pi[(w_\nu-1)h_i].
\end{align*}
The last equality follows from Equation~\eqref{eq:app-score-zero}.
Hence \(a=w_\nu-1\), and its squared scale is
\[
\E_\pi[a^2]
=\E_x\!\left[\chi^2\!\left(\nu(\cdot\mid x)\,\|\,
                               \pi_\theta(\cdot\mid x)\right)\right].
\]

\paragraph{On-policy reinforcement learning.}
For a reward \(r(x,y)\) with no direct parameter dependence,
the score-function identity gives
\(g_i=\partial_i\E_\pi[r]=\E_\pi[rh_i]\).
Subtracting \(\bar r(x)=\E_\pi[r\mid x]\) leaves this expectation unchanged:
\[
g_i=\E_\pi[(r-\bar r(x))h_i],
\qquad
\E_\pi[(r-\bar r(x))^2]=\E_x\Var_\pi(r\mid x).
\]
This establishes the RL row of Table~\ref{tab:fisher-bounds}.

\paragraph{PG-style on-policy distillation.}
The sequence signal is
\[
R_T(x,y)=\sum_t
\log\frac{\pi_T(y_t\mid s_t)}{p_\theta(y_t\mid s_t)}.
\]
The teacher is fixed and assigns positive probability to student-sampled
tokens. Under the stated convention of holding this signal fixed during
the local score-function update, \(g_i=\E_\pi[R_T h_i]\).
Letting \(\bar R_T(x)=\E_\pi[R_T\mid x]\), we obtain
\[
g_i=\E_\pi[(R_T-\bar R_T(x))h_i],
\qquad
\E_\pi[(R_T-\bar R_T(x))^2]=\E_x\Var_\pi(R_T\mid x).
\]

\paragraph{Full-vocabulary on-policy distillation.}
At a sampled prefix \(s_t\), write
\(p_t(a)=p_\theta(a\mid s_t)\) and \(q_t(a)=\pi_T(a\mid s_t)\),
with \(q_t\ll p_t\). Differentiating negative teacher-to-student KL
while holding the sampled prefixes and teacher fixed gives
\begin{equation}
\label{eq:app-full-opd-direct}
g_i=\E_\pi\!\left[\sum_t\sum_a q_t(a)
                              \partial_i\log p_t(a)\right].
\end{equation}
Define
\[
b_t=\frac{q_t(y_t)}{p_t(y_t)}-1,\qquad
A_T=\sum_t b_t,\qquad
\xi_{t,i}=\partial_i\log p_t(y_t),\qquad
h_i=\sum_t\xi_{t,i}.
\]
Both \(b_t\) and \(\xi_{t,i}\) have zero conditional mean given the
history before token \(t\):
\[
\E[b_t\mid s_t]=\sum_a q_t(a)-1=0,
\qquad \E[\xi_{t,i}\mid s_t]=0.
\]
For \(t\ne u\), conditioning on the history of the later token
therefore gives \(\E[b_t\xi_{u,i}]=0\). For the same-token terms,
\[
\E[b_t\xi_{t,i}\mid s_t]
=\sum_a(q_t(a)-p_t(a))\partial_i\log p_t(a)
=\sum_a q_t(a)\partial_i\log p_t(a).
\]
Consequently,
\[
\E_\pi[A_T h_i]
=\E_\pi\!\left[\sum_t b_t\xi_{t,i}\right]=g_i.
\]
This proves the required representation with \(a=A_T\).
The same conditional-mean argument eliminates cross terms in \(A_T^2\):
\[
\E_\pi[A_T^2]
=\E_\pi\!\left[\sum_t
          \chi^2\!\left(q_t\,\|\,p_t\right)\right].
\]
Applying Equation~\eqref{eq:app-common-bound} to each representation
completes the proof of Proposition~\ref{prop:fisher-gradient-bound}.

\subsection{Token-Level Decomposition of Model Fisher}
\label{app:token-fisher-pairwise}

We prove both identities in Proposition~\ref{prop:token-fisher}.
Let \(s_t=(x,y_{<t})\),
\(\xi_{t,i}=\partial_i\log p_\theta(y_t\mid s_t)\), and
\(h_i=\sum_t\xi_{t,i}\).

\paragraph{From sequence Fisher to token Fisher.}
At an active position, on-policy sampling and normalization imply
\[
\E[\xi_{t,i}\mid s_t]
=\sum_a p_\theta(a\mid s_t)\partial_i\log p_\theta(a\mid s_t)=0.
\]
For \(u<t\), the earlier score \(\xi_{u,i}\) is determined by the
history at \(t\). The tower property thus gives
\[
\E[\xi_{u,i}\xi_{t,i}]
=\E[\xi_{u,i}\E[\xi_{t,i}\mid s_t]]=0.
\]
The identity remains valid when scores are zero after termination.
Expanding the square of the sequence score yields
\begin{align*}
F_i
&=\E_\pi\!\left[\left(\sum_t\xi_{t,i}\right)^2\right]\\
&=\E_\pi\!\left[\sum_t\xi_{t,i}^2\right]
  +2\sum_{u<t}\E_\pi[\xi_{u,i}\xi_{t,i}]\\
&=\E_\pi\!\left[\sum_t\E[\xi_{t,i}^2\mid s_t]\right]
=\E_\pi\!\left[\sum_t f_i(s_t)\right].
\end{align*}
The cancellation holds in expectation, not for each sampled response.

\paragraph{Token Fisher as pairwise sensitivity variation.}
Fix a prefix \(s\), write \(p_a=p_\theta(a\mid s)\), and let
\(j_{i,a}=\partial_i z_a(s)\). Differentiating the log softmax gives
\[
\partial_i\log p_a=j_{i,a}-\sum_b p_bj_{i,b}.
\]
With \(\bar j_i=\sum_a p_aj_{i,a}\), this implies
\[
f_i(s)=\sum_a p_a(j_{i,a}-\bar j_i)^2
=\Var_{a\sim p}[j_{i,a}].
\]
For independent \(a,b\sim p\),
\(\E[j_{i,a}j_{i,b}]=\bar j_i^2\). Therefore,
\begin{align*}
\frac12\E_{a,b\overset{\mathrm{iid}}{\sim}p}
             [(j_{i,a}-j_{i,b})^2]
&=\frac12\left(2\E_{a\sim p}[j_{i,a}^2]-2\bar j_i^2\right)\\
&=\Var_{a\sim p}[j_{i,a}]
=f_i(s).
\end{align*}
This proves the second identity and completes
Proposition~\ref{prop:token-fisher}.

%% file: sections/appendix_experiment_settings.tex
\section{Experimental Settings}
\label{app:experiment-settings}

\subsection{Fisher Identifies Where Gradients Are Concentrated}
\label{app:test1-settings}

For Section~\ref{sec:test1}, we measure Qwen3-1.7B-Base at initialization
and at checkpoints from four full-parameter RL training runs. Training
and prompt formatting follow Appendix~\ref{app:test2-settings}.
The Fisher and gradient prompt sets are sampled without replacement from
the training data, are disjoint, and remain fixed across policies.
Each policy generates its own on-policy responses.

\begin{table}[htbp]
\centering
\small
\caption{Test 1 measurement settings.}
\label{tab:app-test1-settings}
\begin{tabularx}{\linewidth}{@{}lX@{}}
\toprule
Setting & Value \\
\midrule
Model & Qwen3-1.7B-Base \\
Policies & Shared Initial; Full training seeds 0--3 at steps
250, 500, 750, and 1000 \\
Prompt source & The 8,513-problem RL training set \\
Fisher sample size & 800 prompts, one response per prompt \\
Gradient sample size & 300 prompts, eight responses per prompt \\
Generation & Temperature 1, top-\(p=1\), no top-\(k\) truncation \\
Maximum prompt / response length & 1,024 / 4,096 tokens \\
Sampling seeds & Prompt selection: 800; Fisher generation: 100;
gradient generation: 200 \\
Selection percentages \(q\) & \(0.05,\ 0.1,\ 0.5,\ 1,\ 5\%\) \\
Precision & FP32 parameters and score accumulation; BF16 computation \\
\bottomrule
\end{tabularx}
\end{table}

\paragraph{Estimators.}
For each response, let
\(h_i(x,y)=\partial_i\sum_t\log p_\theta(y_t\mid x,y_{<t})\),
summing over valid response tokens, including observed stopping tokens.
With \(\bar r_n=8^{-1}\sum_{j=1}^8 r_{nj}\), we compute
\begin{equation}
\label{eq:app-estimators}
\widehat F_i=\frac1{800}\sum_{n=1}^{800}h_i(x_n,y_n)^2,
\qquad
\widehat g_i=\frac1{300\cdot8}
\sum_{n=1}^{300}\sum_{j=1}^{8}
(r_{nj}-\bar r_n)h_i(x_n,y_{nj}).
\end{equation}
Fisher squares each trajectory's score before averaging. The gradient
is averaged with its sign intact, and only the resulting vector is ranked
by absolute value. Neither estimator divides by sampled response length.
The gradient omits the training objective's common fixed-horizon factor,
which does not affect coordinate ranking.

For each \(q\), the highest \(K_q=\lfloor qP/100\rfloor\) coordinates
are selected globally from \(\widehat F\) and \(|\widehat g|\), where
\(P\) counts unique parameter coordinates. Ties use a fixed coordinate
order. Recall is \(100\) times the support intersection divided by
\(K_q\); an equal-size uniform random set has expected recall
\(100K_q/P\%\), approximately \(q\%\).

\subsection{Initial Fisher Masks Sustain Learning}
\label{app:test2-settings}
\label{app:opd-settings}

For Section~\ref{sec:test2}, we compare Full, Fisher, and two random masks
(R0/R1). At each budget, Fisher trains the globally highest
\(K_q\) initial-Fisher coordinates; R0/R1 train exact-size uniform
global subsets with mask seeds 0 and 1. Each mask is fixed throughout
training and shared across training seeds 0--3. Different budgets for a
given random-mask seed use prefixes of the same random permutation.
Full and Initial are shared across budgets within each suite.

RL and OPD use separate initial models and Fisher estimates. Both initial
Fisher estimates use 800 training prompts with one on-policy response
per prompt and the trajectory-score estimator above. Within a suite
and training seed, all conditions share initialization, data-ordering
and sampling rules, optimizer, training budget, and evaluation settings.
Frozen gradients are zeroed after accumulation and before clipping and
the optimizer step.

\begin{table}[htbp]
\centering
\small
\setlength{\tabcolsep}{4pt}
\caption{Training and evaluation settings for the reported Test 2 comparisons.
All conditions within each column share these settings.}
\label{tab:app-test2-settings}
\begin{tabularx}{\linewidth}{@{}lXX@{}}
\toprule
Setting & RL & OPD \\
\midrule
Initial model & Qwen3-1.7B-Base & Qwen3-1.7B (post-trained) \\
Teacher & None & Qwen3-4B-Instruct-2507, frozen \\
Training dataset & SimpleRL-Zoo, Abel levels 3--5 & DAPO-Math-17k \\
Training data size & 8,513 problems & 17,183 problems after deduplication
and removal of exact MATH-500 matches \\
Objective & Dr.\ GRPO & Token-local PG-style distillation \\
Optimizer updates & 1,000 & 20 \\
Prompts per update & 4 & 512 \\
Responses per prompt & 8 & 1 \\
Microbatch size & 4 responses & 2 responses \\
Gradient accumulation & 8 microbatches & 256 microbatches \\
Learning rate & \(5\times10^{-7}\) & \(5\times10^{-6}\) \\
Optimizer & AdamW & AdamW \\
Adam parameters & \(\beta=(0.9,0.999)\), \(\epsilon=10^{-8}\) &
\(\beta=(0.9,0.999)\), \(\epsilon=10^{-8}\) \\
Schedule / weight decay & Constant, no warmup / 0 & Constant, no warmup / 0 \\
Gradient clipping & Global norm 1.0 & Global norm 1.0 \\
Training temperature / top-\(p\) & \(1 / 1\) & \(1 / 1\) \\
Maximum prompt / response tokens & \(1{,}024 / 4{,}096\) &
\(2{,}048 / 8{,}192\) \\
Prompt format & Plain text & Chat, thinking disabled \\
Training seeds & 0, 1, 2, 3 & 0, 1, 2, 3 \\
Reported mask budgets \(q\) & \(0.05,\ 0.1,\ 0.5,\ 1,\ 5\%\) & \(0.5,\ 1,\ 5\%\) \\
\midrule
Evaluation dataset & MATH-500, 500 problems & AIME24, 30 problems \\
Responses per evaluation problem & 4 & 8 \\
Evaluation temperature / top-\(p\) & \(1 / 1\) & \(1 / 0.7\) \\
Evaluation generation seeds & 300--303 & 300--307 \\
Evaluation metrics & Mean response accuracy & Avg@8 and pass@8 \\
\bottomrule
\end{tabularx}
\end{table}

Both suites use FP32 parameters and gradient accumulation, BF16
computation, and no adapters. Each fresh batch receives one optimizer
update. Sampling has no top-\(k\) or min-\(p\) restriction. Evaluation
uses the same prompt format and length limits as training. Observed
stopping tokens are retained; length-truncated responses are not
discarded or given artificial EOS tokens.

\paragraph{RL objective and prompts.}
Rewards are binary mathematical-correctness scores from the shared
answer extraction and \texttt{math-verify} scorer. Dr.\ GRPO subtracts
the prompt group's mean reward without standard-deviation normalization.
The token-summed objective is divided by the response count and the
fixed horizon of 4,096. The clipping parameter is 0.2; there is no
reference-policy KL penalty or entropy bonus. Prompts use
\begin{verbatim}
Question:
{question}
Answer:
Let's think step by step.
\end{verbatim}

\paragraph{OPD objective and prompts.}
The frozen teacher scores tokens sampled from the student. For a batch
with \(N_{\rm tok}\) valid response tokens, the implemented loss is
\begin{equation}
\label{eq:app-opd-loss}
L_{\rm K2}
=\frac1{2N_{\rm tok}}\sum_{n,t\ {\rm valid}}
\left(\log p_\theta(y_{nt}\mid s_{nt})
-\sg[\log\pi_T(y_{nt}\mid s_{nt})]\right)^2.
\end{equation}
Its negative gradient weights each token score by
\(\log\pi_T-\log p_\theta\), using the same whole-batch denominator
for all microbatches. This is a token-local update on fixed sampled
trajectories, rather than the sequence-signal estimator \(R_T h_i\)
in the theoretical PG convention. Answer rewards, group centering,
and PPO clipping are not used in the OPD loss.

Student and teacher receive the same prompt token IDs. The Qwen chat
template is applied with \texttt{enable\_thinking=False} to
{\small
\begin{verbatim}
{question}
Please reason step by step, and put your final answer within \boxed{}.
\end{verbatim}
}
Generation stops at \texttt{<|endoftext|>} or \texttt{<|im\_end|>}.
The OPD evaluation uses AIME I and II from
\texttt{HuggingFaceH4/aime\_2024}.

\paragraph{Evaluation and averaging.}
For \(N\) problems and \(G\) responses per problem, mean response
accuracy is \(100(NG)^{-1}\sum_{n,j}r_{nj}\). Thus RL averages 2,000
binary rewards, and OPD Avg@8 averages 240. OPD pass@8 instead reports
the percentage of questions with at least one correct response among
the eight. All conditions use the same binary scorer.

For each training seed \(s\) and evaluation metric, we compute
\[
M_{R,s}=\frac{M_{R0,s}+M_{R1,s}}2,\qquad
\Delta_s=M_{F,s}-M_{R,s},\qquad
\operatorname{Recovery}_s=
100\,\frac{M_{F,s}-M_0}{M_{\mathrm{Full},s}-M_0}.
\]
Here \(M_0\) is the corresponding suite's shared initial score.
Tables report equal-weight means over the four training seeds.
Recovery is calculated within each seed before averaging, and
\(\Delta\) is in percentage points.

\subsection{Testing the Alignment Between Probability Mass and Sensitivities}
\label{app:alignment-settings}

For the comparison in Section~\ref{sec:why-small-fisher}, Real and
Uniform use the same model and exactly the same prefixes. Only the
distribution used to sample candidate next-token pairs changes.

\begin{table}[htbp]
\centering
\small
\caption{Settings for the Real--Uniform token-Fisher comparison.}
\label{tab:app-alignment-settings}
\begin{tabularx}{\linewidth}{@{}lX@{}}
\toprule
Setting & Value \\
\midrule
Model and policies & Qwen3-1.7B-Base; shared Initial and Full seeds 0--3
at steps 250/500/750/1000 \\
Prompts & 200 prompts sampled without replacement from Test 1's
300 gradient prompts, shared across policies \\
Responses & Each policy's first saved on-policy response per prompt \\
Response generation & Temperature 1, top-\(p=1\), no top-\(k\) truncation;
prompt/response limits 1,024/4,096 \\
Positions & Eight uniformly selected positions per response;
identical for Real and Uniform \\
Token pairs & Eight independent pairs per position per condition,
sampled with replacement \\
Real distribution & The model's full softmax probabilities \\
Uniform distribution & Uniform over the full vocabulary, including special tokens \\
Parameter groups & All; current-policy low-Fisher
99.95\%, 99.9\%, 99.5\%, 99\%, and 95\% \\
Group selection & Each policy's own 800-response sequence-Fisher estimate \\
\bottomrule
\end{tabularx}
\end{table}

\paragraph{Measurement and averaging.}
The selected prompts and positions are fixed before measurement and
are not filtered by correctness. For a parameter group \(\mathcal G\)
and a sampled pair \((a,b)\), we compute
\[
D_{\mathcal G}(s,a,b)
=\frac1{2|\mathcal G|}
 \sum_{i\in\mathcal G}\bigl(j_{i,a}(s)-j_{i,b}(s)\bigr)^2.
\]
This is obtained by backpropagating \((z_a-z_b)/\sqrt2\), squaring
each coordinate gradient, and averaging within the group. Model
parameters are not updated. All six groups share each condition's
pairs and backward calculations. Equal-token pairs contribute zero
and remain in the sample count.

For each condition, we average the eight pairs at each position,
then the eight positions within a response, and finally the 200 prompt
means with equal weight. Each policy therefore uses 1,600 prefixes
and 12,800 pairs per condition. This estimates mean token Fisher over
sampled positions, not the unnormalized sequence Fisher used in Test 1.
Across the four training seeds, Real and Uniform are averaged
separately before taking their ratio; Initial is measured once.
The comparison keeps the original logit sensitivities fixed while
changing both probability concentration and assignment.

%% file: sections/appendix_experiment_results.tex
\section{Additional Experimental Results}
\label{app:experiment-results}

The following results supplement Sections~\ref{sec:test1} and~\ref{sec:test2}.
In the fixed-mask learning tables, entries are equal-weight means over
training seeds 0--3.
R0 and R1 denote the two random masks; Random averages their scores.
Recovery is computed within each training seed before averaging, and
\(\Delta\) denotes Fisher minus Random in percentage points.
Training and evaluation settings are given in
Appendix~\ref{app:test2-settings}.

\subsection{RL: MATH-500}
\label{app:test2-results}

Table~\ref{tab:app-test2-budgets} reports the numerical values for the
MATH-500 comparison in the main text. The score is the mean binary
reward over four responses to each of 500 problems.
Initial scores \(22.65\%\); its evaluation and the Full results are
shared across mask budgets.

\begin{table}[htbp]
\centering
\small
\setlength{\tabcolsep}{3pt}
\caption{RL on MATH-500: four-training-seed mean response accuracy (\%).
The initial model is Qwen3-1.7B-Base. All models are evaluated after
1,000 updates. Recovery is a percentage and \(\Delta\) is in percentage
points.}
\label{tab:app-test2-budgets}
\begin{tabular}{@{}rrrrrrrr@{}}
\toprule
\(q\) (\%) & Full & Fisher & R0 & R1 & Random & Recovery & \(\Delta\) \\
\midrule
0.05 & 56.41 & 49.65 & 22.75 & 23.66 & 23.21 & 79.98 & 26.44 \\
0.1 & 56.41 & 51.41 & 24.36 & 23.44 & 23.90 & 85.21 & 27.51 \\
0.5 & 56.41 & 54.00 & 26.39 & 26.45 & 26.42 & 92.87 & 27.58 \\
1 & 56.41 & 55.03 & 30.93 & 31.21 & 31.07 & 95.90 & 23.96 \\
5 & 56.41 & 56.01 & 47.58 & 47.60 & 47.59 & 98.82 & 8.43 \\
\bottomrule
\end{tabular}
\end{table}

\subsection{OPD: Fisher--Gradient Support Recall}
\label{app:opd-test1-results}

We extend the same-policy Fisher--gradient support comparison in
Section~\ref{sec:test1} to OPD. Table~\ref{tab:app-opd-gradient-recall}
compares the highest-Fisher coordinates with those having the largest
gradient magnitudes at the shared initial policy and checkpoints from
full-parameter OPD training. Checkpoint results average training seeds
0 and 1; the initial result is shared.

\begin{table}[htbp]
\centering
\small
\setlength{\tabcolsep}{6pt}
\caption{OPD Fisher--gradient support recall (\%), supplementing
Section~\ref{sec:test1}. At each policy, recall is the fraction of the
top-\(q\%\) gradient-magnitude coordinates captured by the top-\(q\%\)
Fisher coordinates. Checkpoint rows are equal-weight means over full-training
seeds 0 and 1; the shared initial policy is measured once.
Random gives the analytic expectation for a uniform subset of the same size.}
\label{tab:app-opd-gradient-recall}
\begin{tabular}{@{}lrrr@{}}
\toprule
Checkpoint & \(q=0.5\%\) & \(q=1\%\) & \(q=5\%\) \\
\midrule
Initial (step 0) & 37.12 & 35.66 & 37.84 \\
Step 5 & 43.26 & 41.05 & 41.90 \\
Step 10 & 47.64 & 45.63 & 47.64 \\
Step 15 & 47.75 & 45.80 & 48.25 \\
Step 20 & 49.98 & 47.86 & 50.11 \\
\midrule
Random & 0.50 & 1.00 & 5.00 \\
\bottomrule
\end{tabular}
\end{table}

\subsection{OPD: AIME24}
\label{app:opd-results}

We report AIME24 using the OPD protocol's chat non-thinking evaluation,
which matches the training prompt format. Table~\ref{tab:app-opd-aime24}
contains the completed four-seed comparisons at \(q=0.5\%\), \(1\%\), and \(5\%\).
Avg@8 averages all \(30\times8\) binary rewards; pass@8 is the percentage
of problems answered correctly at least once among the eight responses.
The shared initial Qwen3-1.7B student scores \(16.25\%\) Avg@8 and
\(36.67\%\) pass@8. These OPD scores use a different initial model and
evaluation protocol from the RL results above.

\begin{table}[htbp]
\centering
\footnotesize
\setlength{\tabcolsep}{3pt}
\caption{OPD on AIME24: means over four training seeds after 20 updates.
Scores and Recovery are percentages; \(\Delta\) is in percentage points.
Both metrics use the same eight responses per problem.}
\label{tab:app-opd-aime24}
\begin{tabular}{@{}rlrrrrrrr@{}}
\toprule
\(q\) (\%) & Metric & Full & Fisher & R0 & R1 & Random & Recovery & \(\Delta\) \\
\midrule
0.5 & Avg@8 & 26.67 & 25.10 & 14.69 & 15.42 & 15.05 & 89.49 & 10.05 \\
0.5 & pass@8 & 47.50 & 52.50 & 33.33 & 30.83 & 32.08 & 156.25 & 20.42 \\
\midrule
1 & Avg@8 & 26.67 & 26.15 & 15.00 & 16.46 & 15.73 & 99.47 & 10.42 \\
1 & pass@8 & 47.50 & 55.00 & 29.17 & 38.33 & 33.75 & 185.42 & 21.25 \\
\midrule
5 & Avg@8 & 26.67 & 26.25 & 24.27 & 23.33 & 23.80 & 106.59 & 2.45 \\
5 & pass@8 & 47.50 & 50.83 & 49.17 & 45.83 & 47.50 & 147.92 & 3.33 \\
\bottomrule
\end{tabular}
\end{table}

%% file: main_arxiv.bbl
\begin{thebibliography}{24}
\providecommand{\natexlab}[1]{#1}
\providecommand{\url}[1]{\texttt{#1}}
\expandafter\ifx\csname urlstyle\endcsname\relax
  \providecommand{\doi}[1]{doi: #1}\else
  \providecommand{\doi}{doi: \begingroup \urlstyle{rm}\Url}\fi

\bibitem[Adewuyi et~al.(2026)Adewuyi, Okibe, and Ivanov]{adewuyi2026multiple}
Israel Adewuyi, Solomon Okibe, and Vladmir Ivanov.
\newblock The multiple ticket hypothesis: Random sparse subnetworks suffice for
  {RLVR}.
\newblock \emph{arXiv preprint arXiv:2602.01599}, 2026.
\newblock URL \url{https://arxiv.org/abs/2602.01599}.

\bibitem[Agarwal et~al.(2024)Agarwal, Vieillard, Zhou, Stanczyk, Garea, Geist,
  and Bachem]{agarwal2024onpolicy}
Rishabh Agarwal, Nino Vieillard, Yongchao Zhou, Piotr Stanczyk, Sabela~Ramos
  Garea, Matthieu Geist, and Olivier Bachem.
\newblock On-policy distillation of language models: Learning from
  self-generated mistakes.
\newblock In \emph{International Conference on Learning Representations}, 2024.
\newblock URL
  \url{https://proceedings.iclr.cc/paper_files/paper/2024/hash/5be69a584901a26c521c2b51e40a4c20-Abstract-Conference.html}.

\bibitem[{DeepSeek-AI}(2025)]{deepseekai2025r1}
{DeepSeek-AI}.
\newblock {DeepSeek-R1}: Incentivizing reasoning capability in {LLMs} via
  reinforcement learning.
\newblock \emph{arXiv preprint arXiv:2501.12948}, 2025.
\newblock URL \url{https://arxiv.org/abs/2501.12948}.

\bibitem[Geva et~al.(2022)Geva, Caciularu, Wang, and
  Goldberg]{geva2022transformer}
Mor Geva, Avi Caciularu, Kevin Wang, and Yoav Goldberg.
\newblock Transformer feed-forward layers build predictions by promoting
  concepts in the vocabulary space.
\newblock In \emph{Proceedings of the 2022 Conference on Empirical Methods in
  Natural Language Processing}, pp.\  30--45. Association for Computational
  Linguistics, 2022.
\newblock \doi{10.18653/v1/2022.emnlp-main.3}.

\bibitem[Gurnee et~al.(2024)Gurnee, Horsley, Guo, Kheirkhah, Sun, Hathaway,
  Nanda, and Bertsimas]{gurnee2024universal}
Wes Gurnee, Theo Horsley, Zifan~Carl Guo, Tara~Rezaei Kheirkhah, Qinyi Sun,
  Will Hathaway, Neel Nanda, and Dimitris Bertsimas.
\newblock Universal neurons in {GPT2} language models.
\newblock \emph{arXiv preprint arXiv:2401.12181}, 2024.

\bibitem[Lightman et~al.(2023)Lightman, Kosaraju, Burda, Edwards, Baker, Lee,
  Leike, Schulman, Sutskever, and Cobbe]{lightman2023lets}
Hunter Lightman, Vineet Kosaraju, Yura Burda, Harri Edwards, Bowen Baker, Teddy
  Lee, Jan Leike, John Schulman, Ilya Sutskever, and Karl Cobbe.
\newblock Let's verify step by step.
\newblock \emph{arXiv preprint arXiv:2305.20050}, 2023.

\bibitem[Liu et~al.(2025)Liu, Chen, Li, Qi, Pang, Du, Lee, and
  Lin]{liu2025r1zero}
Zichen Liu, Changyu Chen, Wenjun Li, Penghui Qi, Tianyu Pang, Chao Du, Wee~Sun
  Lee, and Min Lin.
\newblock Understanding {R1-Zero}-like training: A critical perspective.
\newblock \emph{arXiv preprint arXiv:2503.20783}, 2025.
\newblock URL \url{https://arxiv.org/abs/2503.20783}.

\bibitem[Luo et~al.(2026)Luo, Wang, Zhou, Ye, Zhao, Zhang, and
  Wei]{luo2026rlforgets}
Mao-Lin Luo, Zhe-Xu Wang, Zi-Hao Zhou, Bo~Ye, Jian Zhao, Min-Ling Zhang, and
  Tong Wei.
\newblock {RL} forgets! {Towards} continual policy optimization.
\newblock \emph{arXiv preprint arXiv:2607.04364}, 2026.

\bibitem[Martens(2020)]{martens2020natural}
James Martens.
\newblock New insights and perspectives on the natural gradient method.
\newblock \emph{Journal of Machine Learning Research}, 21\penalty0
  (146):\penalty0 1--76, 2020.
\newblock URL \url{https://www.jmlr.org/papers/v21/17-678.html}.

\bibitem[Miahi \& Belilovsky(2026)Miahi and Belilovsky]{miahi2026understanding}
Erfan Miahi and Eugene Belilovsky.
\newblock Understanding and exploiting weight update sparsity for
  communication-efficient distributed {RL}.
\newblock arXiv preprint arXiv:2602.03839, 2026.

\bibitem[Mukherjee et~al.(2025)Mukherjee, Yuan, Hakkani-Tur, and
  Peng]{mukherjee2025subnetworks}
Sagnik Mukherjee, Lifan Yuan, Dilek Hakkani-Tur, and Hao Peng.
\newblock Reinforcement learning finetunes small subnetworks in large language
  models.
\newblock In \emph{Advances in Neural Information Processing Systems},
  volume~38, 2025.
\newblock \doi{10.52202/085713-4399}.
\newblock URL
  \url{https://proceedings.neurips.cc/paper_files/paper/2025/hash/bf235a1d6780afd979f2f81676f43413-Abstract-Conference.html}.

\bibitem[Mukherjee et~al.(2026)Mukherjee, Yuan, Jayasinha, Hakkani-T{\"u}r, and
  Peng]{mukherjee2026we}
Sagnik Mukherjee, Lifan Yuan, Pavan Jayasinha, Dilek Hakkani-T{\"u}r, and Hao
  Peng.
\newblock Do we need {Adam}? {Surprisingly} strong and sparse reinforcement
  learning with {SGD} in {LLMs}.
\newblock \emph{arXiv preprint arXiv:2602.07729}, 2026.
\newblock \doi{10.48550/arXiv.2602.07729}.
\newblock URL \url{https://arxiv.org/abs/2602.07729}.

\bibitem[Ouyang et~al.(2022)Ouyang, Wu, Jiang, Almeida, Wainwright, Mishkin,
  Zhang, Agarwal, Slama, Ray, Schulman, Hilton, Kelton, Miller, Simens, Askell,
  Welinder, Christiano, Leike, and Lowe]{ouyang2022instructgpt}
Long Ouyang, Jeffrey Wu, Xu~Jiang, Diogo Almeida, Carroll~L. Wainwright, Pamela
  Mishkin, Chong Zhang, Sandhini Agarwal, Katarina Slama, Alex Ray, John
  Schulman, Jacob Hilton, Fraser Kelton, Luke Miller, Maddie Simens, Amanda
  Askell, Peter Welinder, Paul~F. Christiano, Jan Leike, and Ryan Lowe.
\newblock Training language models to follow instructions with human feedback.
\newblock In \emph{Advances in Neural Information Processing Systems},
  volume~35, 2022.
\newblock \doi{10.52202/068431-2011}.
\newblock URL
  \url{https://proceedings.neurips.cc/paper/2022/hash/b1efde53be364a73914f58805a001731-Abstract-Conference.html}.

\bibitem[Peters et~al.(2005)Peters, Vijayakumar, and Schaal]{peters2005natural}
Jan Peters, Sethu Vijayakumar, and Stefan Schaal.
\newblock Natural actor-critic.
\newblock In \emph{Machine Learning: ECML 2005}, pp.\  280--291. Springer,
  2005.

\bibitem[Raykov \& Veiga(2026)Raykov and Veiga]{raykov2026information}
Yordan Raykov and Rodrigo Veiga.
\newblock Information-geometric forward policy training in {GFlowNets}.
\newblock \emph{arXiv preprint arXiv:2608.03967}, 2026.

\bibitem[Rios et~al.(2025)Rios, Dognin, Luss, and
  Natesan~Ramamurthy]{rios2025sparsity}
Jesus Rios, Pierre Dognin, Ronny Luss, and Karthikeyan Natesan~Ramamurthy.
\newblock Sparsity may be all you need: Sparse random parameter adaptation.
\newblock In \emph{Findings of the Association for Computational Linguistics:
  EMNLP 2025}, pp.\  18650--18666, 2025.
\newblock \doi{10.18653/v1/2025.findings-emnlp.1013}.
\newblock URL \url{https://aclanthology.org/2025.findings-emnlp.1013/}.

\bibitem[Shao et~al.(2024)Shao, Wang, Zhu, Xu, Song, Bi, Zhang, Zhang, Li, Wu,
  and Guo]{shao2024deepseekmath}
Zhihong Shao, Peiyi Wang, Qihao Zhu, Runxin Xu, Junxiao Song, Xiao Bi, Haowei
  Zhang, Mingchuan Zhang, Y.~K. Li, Y.~Wu, and Daya Guo.
\newblock {DeepSeekMath}: Pushing the limits of mathematical reasoning in open
  language models.
\newblock \emph{arXiv preprint arXiv:2402.03300}, 2024.
\newblock URL \url{https://arxiv.org/abs/2402.03300}.

\bibitem[Sharma et~al.(2024)Sharma, Muralidhar, Xu, Yousuf, and
  Ramakrishnan]{sharma2024information}
Mandar Sharma, Nikhil Muralidhar, Shengzhe Xu, Raquib~Bin Yousuf, and Naren
  Ramakrishnan.
\newblock Information guided regularization for fine-tuning language models.
\newblock In \emph{Conference on Language Modeling}, 2024.

\bibitem[Shenfeld et~al.(2026)Shenfeld, Pari, and Agrawal]{shenfeld2025razor}
Idan Shenfeld, Jyothish Pari, and Pulkit Agrawal.
\newblock {RL}'s razor: Why online reinforcement learning forgets less.
\newblock In \emph{International Conference on Learning Representations}, 2026.
\newblock URL \url{https://openreview.net/forum?id=7HNRYT4V44}.

\bibitem[Sung et~al.(2021)Sung, Nair, and Raffel]{sung2021fish}
Yi-Lin Sung, Varun Nair, and Colin~A. Raffel.
\newblock Training neural networks with fixed sparse masks.
\newblock In \emph{Advances in Neural Information Processing Systems},
  volume~34, pp.\  24193--24205, 2021.
\newblock URL
  \url{https://proceedings.neurips.cc/paper/2021/hash/cb2653f548f8709598e8b5156738cc51-Abstract.html}.

\bibitem[Xu et~al.(2021)Xu, Luo, Zhang, Tan, Chang, Huang, and
  Huang]{xu2021childtuning}
Runxin Xu, Fuli Luo, Zhiyuan Zhang, Chuanqi Tan, Baobao Chang, Songfang Huang,
  and Fei Huang.
\newblock Raise a child in large language model: Towards effective and
  generalizable fine-tuning.
\newblock In \emph{Proceedings of the 2021 Conference on Empirical Methods in
  Natural Language Processing}, pp.\  9514--9528, 2021.
\newblock \doi{10.18653/v1/2021.emnlp-main.749}.
\newblock URL \url{https://aclanthology.org/2021.emnlp-main.749/}.

\bibitem[Yang et~al.(2025)Yang, Li, Yang, Zhang, Hui, Zheng,
  et~al.]{yang2025qwen3}
An~Yang, Anfeng Li, Baosong Yang, Beichen Zhang, Binyuan Hui, Bo~Zheng, et~al.
\newblock {Qwen3} technical report.
\newblock \emph{arXiv preprint arXiv:2505.09388}, 2025.
\newblock URL \url{https://arxiv.org/abs/2505.09388}.

\bibitem[Yu et~al.(2026)Yu, Liu, Hu, Ma, Jiang, and Ye]{yu2026dense}
Guo Yu, Wenlin Liu, Yulan Hu, Hao-Xuan Ma, Jun-Peng Jiang, and Han-Jia Ye.
\newblock Dense supervision, sparse updates: On the sparsity and geometry of
  on-policy distillation.
\newblock \emph{arXiv preprint arXiv:2606.13657}, 2026.
\newblock URL \url{https://arxiv.org/abs/2606.13657}.

\bibitem[Zhu et~al.(2025)Zhu, Zhang, Huang, Su, Liu, Zhao, Fedorov, Pirsiavash,
  Sha, Lee, Pan, Wang, Tian, and Tai]{zhu2025path}
Hanqing Zhu, Zhenyu Zhang, Hanxian Huang, DiJia Su, Zechun Liu, Jiawei Zhao,
  Igor Fedorov, Hamed Pirsiavash, Zhizhou Sha, Jinwon Lee, David~Z. Pan,
  Zhangyang Wang, Yuandong Tian, and Kai~Sheng Tai.
\newblock The path not taken: {RLVR} provably learns off the principals.
\newblock \emph{arXiv preprint arXiv:2511.08567}, 2025.
\newblock URL \url{https://arxiv.org/abs/2511.08567}.

\end{thebibliography}
